\documentclass{article}

\usepackage[preprint,nonatbib]{neurips_2024}

\usepackage[utf8]{inputenc} 
\usepackage[T1]{fontenc}    
\usepackage[colorlinks,linkcolor=mydarkred,citecolor=mydarkgreen]{hyperref}
\usepackage{url}            
\usepackage{booktabs}       
\usepackage{amsfonts}       
\usepackage{nicefrac}       
\usepackage{microtype}      
\usepackage{xcolor}         

\PassOptionsToPackage{numbers, compress}{natbib}

\usepackage{wrapfig}

\definecolor{mydarkred}{rgb}{0.6,0,0}
\definecolor{mydarkgreen}{rgb}{0,0.6,0}
\newcommand{\cmark}{\ding{51}}%
\newcommand{\xmark}{\ding{55}}%
\usepackage{booktabs}       
\usepackage{amsfonts} 
\usepackage{amssymb}
\usepackage{nicefrac}       
\usepackage{microtype}      
\usepackage{algorithmic}
\usepackage{algorithm}
\usepackage[algo2e,linesnumbered]{algorithm2e}
\usepackage[pdftex]{graphicx}
\usepackage{multirow}
\usepackage{epstopdf}
\usepackage{multicol}
\usepackage{bm}
\usepackage{amssymb, amsthm}
\usepackage{caption}
\usepackage{subfigure}
\usepackage{enumitem}
\usepackage{amsmath}
\newtheorem{theorem}{Theorem}
\newtheorem{lemma}{Lemma}

\newcommand{\squishlist}{
\begin{list}{{{\small{$\bullet$}}}}
{\setlength{\itemsep}{3pt}      \setlength{\parsep}{1pt}
\setlength{\topsep}{1pt}       \setlength{\partopsep}{0pt}
\setlength{\leftmargin}{1em} \setlength{\labelwidth}{1em}
\setlength{\labelsep}{0.5em} } }
\newcommand{\squishend}{  \end{list}  }
\makeatletter
\newcommand*{\rom}[1]{\expandafter\@slowromancap\romannumeral #1@}

\makeatother
\usepackage{pifont}
\usepackage{amssymb}
\usepackage{adjustbox}
\usepackage{mathtools}
\usepackage{sidecap}
\usepackage{bbm}

\title{LayerMatch: Do Pseudo-labels Benefit All Layers?}

\author{
\textbf{Chaoqi Liang}$^{1,2}$, \textbf{Guanglei Yang}$^1$, \textbf{Lifeng Qiao}$^{2,3}$, \textbf{Zitong Huang}$^1$,\\ 
\textbf{Hongliang Yan}$^2$, \textbf{Yunchao Wei}$^4$, \textbf{Wangmeng Zuo}$^1$ \\
\small{$^1$Harbin Institute of Technology,\, $^2$Shanghai AI Laboratory,\, $^3$Shanghai Jiao Tong University,} \\
\small{$^4$Beijing Jiaotong University}\\
\small{\texttt{lcqfacai@outlook.com}}
}

\begin{document}

\maketitle

\begin{abstract}
Deep neural networks have achieved remarkable performance across various tasks when supplied with large-scale labeled data. However, the collection of labeled data can be time-consuming and labor-intensive. Semi-supervised learning (SSL), particularly through pseudo-labeling algorithms that iteratively assign pseudo-labels for self-training, offers a promising solution to mitigate the dependency of labeled data. Previous research generally applies a uniform pseudo-labeling strategy across all model layers, assuming that pseudo-labels exert uniform influence throughout. Contrasting this, our theoretical analysis and empirical experiment demonstrate feature extraction layer and linear classification layer have distinct learning behaviors in response to pseudo-labels. Based on these insights, we develop two layer-specific pseudo-label strategies, termed Grad-ReLU and Avg-Clustering. Grad-ReLU mitigates the impact of noisy pseudo-labels by removing the gradient detrimental effects of pseudo-labels in the linear classification layer. Avg-Clustering accelerates the convergence of feature extraction layer towards stable clustering centers by integrating consistent outputs. Our approach, LayerMatch, which integrates these two strategies, can avoid the severe interference of noisy pseudo-labels in the linear classification layer while accelerating the clustering capability of the feature extraction layer. Through extensive experimentation, our approach consistently demonstrates exceptional performance on standard semi-supervised learning benchmarks, achieving a significant improvement of 10.38\% over baseline method and a 2.44\% increase compared to state-of-the-art methods.
\end{abstract}

\section{Introduction}
Deep learning has demonstrated remarkable performance across various domains, benefiting from extensive labeled datasets~\cite{he2016deep,vaswani2017attention,dong2018speech}.
However, acquiring such comprehensive labeled datasets poses substantial challenges due to the high costs and labor-intensive processes.
In this context, semi-supervised learning (SSL) emerges as a pivotal approach, enabling models to extract valuable information from a wealth of unlabeled data, thereby reducing the dependency on labeled data \cite{xiaojin2008semi,zhu2009introduction,sohn2020fixmatch}. 
Within SSL, two prominent techniques have shown significant effectiveness: pseudo-labeling and consistency regularization. Specifically, pseudo-labeling~\cite{lee2013pseudo, xie2020self} leverages labeled data to assign pseudo-labels to unlabeled data. This approach allows the model to learn from a larger dataset without requiring extensive labeling efforts. In addition, consistency regularization~\cite{bachman2014learning, sajjadi2016regularization, laine2016temporal} ensures that the model produces consistent outputs for the same input, even when the input undergoes various data augmentations. Consequently, this approach improves the model's generalization ability.

Recently, integrating pseudo-labeling with consistency regularization has emerged as a predominant approach. A prime example is FixMatch~\cite{sohn2020fixmatch}, which uses weak and strong data augmentation to combine pseudo-labeling and consistency regularization.
In FixMatch, pseudo-labels are generated by selecting predictions from weakly augmented inputs that exceed a confidence threshold (i.e., $\tau=0.95$). It then calculates the cross-entropy loss between the strongly data augmentation of these inputs and their corresponding pseudo-labels as the consistency regularization.
Despite the inherent inaccuracies in pseudo-labels, FixMatch lets the deep models learn effectively and improve performance.
Figure~\ref{fig1} illustrates the training process in a classic semi-supervised learning (SSL) setup. FixMatch (yellow curve) significantly outperforms the only supervised learning method using labeled data (red dashed line), highlighting the effectiveness of pseudo-labeling combined with consistency regularization in leveraging unlabeled data to enhance model performance.

\begin{wrapfigure}{R}
{0.55\columnwidth}
\centering
\vspace{-6mm}
\includegraphics[width=\linewidth]{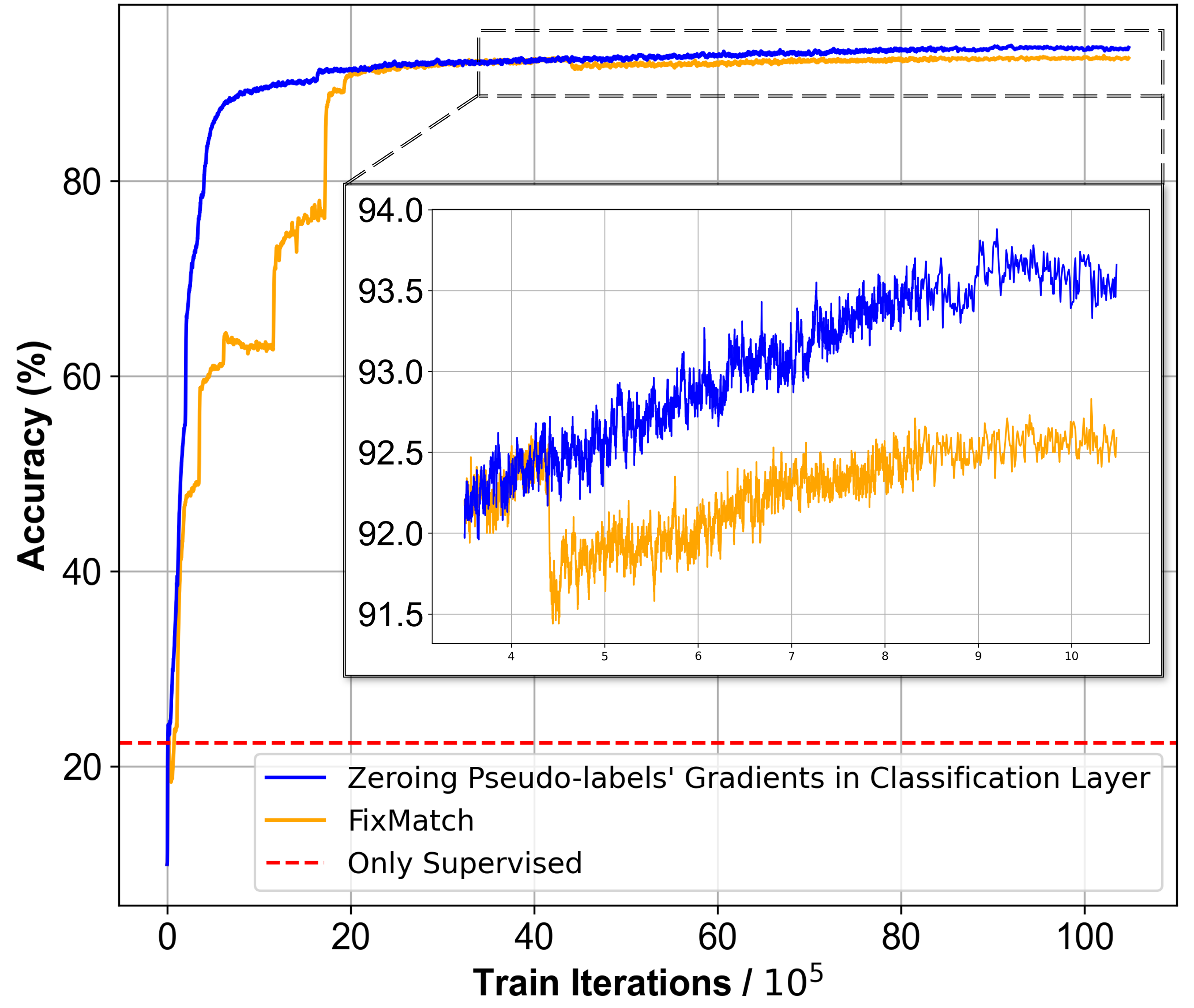}
\vspace{-6mm}
\caption{
The training curves on CIFAR-10 with a classic semi-supervised learning (SSL) setup. In this setup, CIFAR-10 dataset includes 50,000 training samples, with only 40 labeled and 49,960 unlabeled samples. Models are evaluated on a 10,000 samples test set.}
\label{fig1}
\vspace{-4mm}
\end{wrapfigure}
Despite the strengths of FixMatch, an anomalous behavior is observed in its training curve. As illustrated in Figure~\ref{fig1},
FixMatch does not gradually converge to the accuracy upper limit (yellow dashed line); instead, it exhibits unstable oscillations and sudden performance drops. This phenomenon suggests potential instability and a performance ceiling inherent in the model's structure.
Through meticulous research, we identified a pivotal, yet implicit, assumption underlying previous works \cite{sohn2020fixmatch,zhang2021flexmatch,wang2022freematch,chen2023softmatch} that all layers of a deep learning model would uniformly benefit from pseudo-labels. 
Our comprehensive investigations, including both formal methods and empirical experiments, challenge this assumption by revealing that \textit{NOT} all layers benefit from the pseudo-labels.
As shown in Figure \ref{fig1}, just zeroing pseudo-labels' gradients in 
classification layer markedly enhances model performance. 
This discovery on experiment opposes the prevailing notion and underscores the necessity of adopting layer-specific approaches when applying pseudo-labels within SSL frameworks.

This study leverages consistency regularization loss function to conduct a formal analysis, separating the linear classification layer from the feature extraction layer. We identified a significant discrepancy between these two layers in their pseudo-label learning behaviors.
Our findings delineate two key patterns: (1) The feature extraction layer enhances data clustering by optimizing the consistency regularization loss. This process causes well-optimized data to cluster tightly, forming high-density regions, while poorly optimized data remains scattered, creating low-density regions that contain many incorrect pseudo-labels. (2) Because the consistency regularization loss in low-density regions is difficult for the feature extraction layer to optimize, the linear classification layer must compensate by optimizing the loss for this data. So, the abundance of incorrect pseudo-labels in these low-density regions significantly impacts the linear classification layer, thereby undermining performance.

To accommodate the different learning patterns between layers, we propose LayerMatch. LayerMatch employs a layered pseudo-label strategy using Grad-ReLU and Avg-Clustering. Grad-ReLU zeros pseudo-label gradients in the linear classification layer while retaining them in the feature extraction layer. This reduces the impact of pseudo-label inaccuracies on the linear classification layer and maintains the learning integrity of the feature extraction layer. Avg-Clustering uses an exponential moving average strategy to stabilize the clustering centers of the feature extraction layer. This promotes high-density clustering, accelerates convergence, and reduces the impact of pseudo-label errors in low-density regions.

The contributions of this work are summarized as follows:
(1) We identify the discrepancy learning behaviors of pseudo-labels between the feature extraction layer and the linear classification layer in SSL. Pseudo-labels are beneficial for the feature extraction layer but detrimental to the linear classification layer.
(2) We propose LayerMatch, a novel strategy for layer-specific application of pseudo-labels that maintains the benefits of pseudo-labels for the feature extraction layer while mitigating their detrimental effects on the linear classification layer.
(3) We conduct extensive experiments and validate that LayerMatch achieves an average boost of $2.44\%$ against state-of-the-art methods on standard datasets and $10.38\%$ against FixMatch.
\section{Related Work}
\subsection{Semi-supervised Learning Methods} 
Semi-supervised learning is a widely researched field. Through numerous review papers \cite{van2020survey, zhu2005semi, ouali2020overview, yang2022survey, han2024deep, huang2024modelling}, it is evident that the research paradigms and methodological approaches are extremely diverse. For instance, ``transductive'' models \cite{gammerman2013learning,joachims2003transductive,joachims1999transductive, ge2023semi, li2023knowledge}, graph-based methods \cite{zhu2003semi,bengio2006label,liu2019deep, wan2021contrastive, song2024optimal}, generative modeling \cite{belkin2001laplacian,lasserre2006principled,hinton2007using,coates2011importance,goodfellow2012spike,kingma2014semi,pu2016variational,odena2016semi,salimans2016improved,li2021semantic,miao2021generative,wan2021contrastive,you2024diffusion} have been explored. This paper focuses on discussing the latest methods and the most widely used approaches. \textbf{Pseudo-Labeling} \cite{lee2013pseudo} generates artificial labels for unlabeled data and trains the model in a self-training manner. However, pseudo-labels often contain numerous errors, which can easily accumulate during the training process. A series of works have proposed filtering which pseudo-labels to use for self-training. FixMatch \cite{sohn2020fixmatch} adopts a fixed threshold to filter pseudo-labels. FlexMatch \cite{zhang2021flexmatch}, Freematch \cite{wang2022freematch}, and InstanT \cite{li2024instant} each employ a mechanism to set a dynamic threshold at different granular levels from dataset to example. Multi-Head Co-Training \cite{chen2022semi} proposes integrating multiple heads to generate higher-quality pseudo-labels. Some works \cite{ren2020not,chen2023softmatch} assign different weights to pseudo-labels based on confidence. \textbf{Consistency regularization} \cite{bachman2014learning, sajjadi2016regularization, laine2016temporal} is to make the model produce similar predictions for different perturbations of the same data. Many works \cite{sohn2020fixmatch,zhang2021flexmatch,wang2022freematch,chen2023softmatch,berthelot2019mixmatch,berthelot2019remixmatch} focus on consistency regularization for pseudo-labels. Subsequent improvements include adding adversarial learning to consistency regularization in Debiased \cite{chen2022debiased}, adding perturbations to the network for consistency regularization in FlatMatch \cite{huang2024flatmatch}, and incorporating posterior information into consistency regularization in InfoMatch \cite{han2024infomatch}. ReFixMatch \cite{nguyen2023boosting} utilizes low-confidence predictions for consistency regularization.
\subsection{Data Selection}
Not all data is beneficial to the model. 
Data selection is to evaluate and choose data that is beneficial to the model for training. 
In the domain of data selection, three primary methodologies stand out: \textit{retraining-based}, \textit{gradient-based}, and \textit{data-centric learning}. Retraining-based methods like the leave-one-out approach and Shapley values \cite{cook1982residuals, ghorbani2019data, jia2019towards, kwon2022beta} are impractical for deep learning due to their high computational demands. Gradient-based methods \cite{koh2017understanding, agarwal2017second, yeh2018representer, chen2021hydra, schioppa2022scaling, feldman2020neural, kwon2023datainf, pruthi2020estimating, charpiat2019input, kong2021resolving, grosse2023studying, schioppa2024theoretical} offer a more feasible alternative by approximating influence using derivatives, though they often assume model convexity. Data-centric learning involves strategies like datamodels, data efficiency, data pruning, model pruning, antidote data augmentation, feature selection, and active learning \cite{ilyas2022datamodels, jain2022efficient, paul2021deep, killamsetty2021retrieve, tan2024data, yang2022dataset, lyu2023deeper, chhabra2022fair, li2022learning, hall1999correlation, cai2018feature, cohn1996active, isal, nguyen2022measure, wei2015submodularity, solans2021poisoning, mehrabi2021exacerbating, chhabra2022robust, dai2023training}, which aim to optimize training data usage, enhance model robustness, or accelerate training processes. Anshuman et al. \cite{chhabra2023data} proposes data selection methods for Vision Transformers (ViT) \cite{dosovitskiy2020image} through influence functions.

\section{Preliminaries}

We formulate the framework of semi-supervised learning (SSL) for a $\mathcal{C}$-class classification problem. During the training stage, labeled and unlabeled data batches are randomly sampled from their respective datasets. The labeled data batch is denoted as $\mathcal{D}_{L} :=\left\{\mathbf{x}_{i}^l, \mathbf{y}_{i}^l \right\}_{i = 1}^{\mathcal{B}_L}$, and the unlabeled data batch is denoted as $\mathcal{D}_{U}:=\left\{\mathbf{x}_{i}^u \right\}_{i =1}^{\mathcal{B}_U}$. Meanwhile, $\mathcal{B}_L$ and $\mathcal{B}_U$ represent the batch sizes of the labeled and unlabeled data, 
respectively. To effectively train the model, SSL methods generally employ the supervised loss function for labeled data as follows:
\begin{equation}
\label{Ls}
    \mathcal{L}_s = \frac{1}{\mathcal{B}_L} \sum_{i=1}^{\mathcal{B}_L} \mathcal{H}\Bigg( \mathbf{y}_{i}^l, \mathcal{P}_{\boldsymbol{\mathrm{\Theta_t}},\boldsymbol{\beta_t}}\bigg( y \,\Big|\, \omega\big( \mathbf{x}_{i}^l \big) \bigg) \Bigg),
\end{equation}
where $\mathcal{H}(\cdot, \cdot)$ refers to cross-entropy loss,  $\omega(\cdot)$ donates the stochastic (weak) data augmentation function (i.e., random crop and flip), and $t$ means during $t$-th iteration. $\mathcal{P}_{\boldsymbol{\mathrm{\Theta_t}} ,\boldsymbol{\beta_t}}(\mathbf{y}|\mathbf{x}) \in \mathbb{R}^C$ donates the model's prediction. $\mathrm{\Theta_t}$ and $\boldsymbol{\beta_t}$ are the parameters of the feature extraction layer and the linear classification layer. 

For unlabeled data, a confidence threshold $\tau$ is chosen to select and combine unlabeled data into pseudo-labels, denoted as $\mathcal{D}_{\tau} := \{ (\mathbf{x}_{i}^u, \mathbf{y}_{i}^u) \, | \, \mathcal{P}_{\boldsymbol{\mathrm{\Theta_t}},\boldsymbol{\beta_t}}\big(\mathbf{y}_{i}^u = \mathbf{\hat{p}}_i \, | \, \omega(\mathbf{x}_{i}^u)) \geq \tau \}$, where $\mathbf{\hat{p}}_i$ is the one-hot label of $\operatorname{argmax}(\mathcal{P}_{\boldsymbol{\mathrm{\Theta_t}},\boldsymbol{\beta_t}}(\mathbf{y}|\omega(\mathbf{x}^u_i)))$. The unsupervised loss function for unlabeled data is calculated using cross-entropy with pseudo-labels:
\begin{equation}
\label{LU}
\mathcal{L}_u = \frac{1}{\left| \mathcal{D}_{\tau} \right|} \sum_{i=1}^{\left| \mathcal{D}_{\tau} \right|} \mathcal{H}\Bigg(\mathbf{y}_{i}^u, \mathcal{P}_{\boldsymbol{\mathrm{\Theta_t}},\boldsymbol{\beta_t}}\bigg( y \,\Big|\, \Omega\left( \mathbf{x}_{i}^u \right) \bigg) \Bigg),
\end{equation}
where $\Omega(\cdot)$ means the strong data augmentation function (i.e.,
RandAugment \cite{cubuk2020randaugment}) and $\left| \mathcal{D}_{\tau} \right|$ represents the number of pseudo-labels.

\section{Method}
\subsection{Motivation}
This subsection outlines the steps to derive Theorem \ref{minichange}. By analyzing the conclusions of Theorem \ref{minichange}, we gain insights into designing hierarchical pseudo-labeling strategies to improve SSL algorithms.


For clarity in discussion, this paper uses the One-Vs-Rest (OVR) strategy~\cite{biship2007pattern, murphy2012machine} to transform the $\mathcal{C}$-category classification problem into a binary classification problem. 
Without loss of generality, $\mathcal{C}$ is set to 2 for the following discussion in this subsection.

In the case of binary classification, the softmax classifier 
reduces to the sigmoid classifier. 
Instead of calculating probabilities for multiple classes, the model calculates the probability of a single class. $\mathcal{D} = \{\mathbf{x} \mid \mathbf{x} \in \mathbb{R}^n\}$ represents 
the bounded and compact input dataset, where each dimension of $\mathbf{x}$ is independent. $\mathcal{D}_{\boldsymbol{k}}$ represents the set of all $\mathbf{x}$ with label $\boldsymbol{k}$, where \(k \in \{0, 1\}\). 
Both \(\Omega(\mathbf{x})\) and \( \omega(\mathbf{x}) \) belong to \( \mathcal{D}_{\boldsymbol{k}} \), if and only if $\mathbf{x} \in \mathcal{D}_{\boldsymbol{k}}$. 
The expression for binary 
logistic regression on the output features of the feature extraction layer is given by:
\begin{align}
\label{p_Theta}
\mathcal{P}_{\boldsymbol{\mathrm{\Theta_t}}, \boldsymbol{\beta_t}}(\mathbf{x}) = \frac{1}{1 + e^{\boldsymbol{\beta_t}^T \mathcal{M}_{\mathrm{\Theta_t}}(\mathbf{x})}},
\end{align}
where $\boldsymbol{\beta_t} \in \mathbb{R}^n$ denotes the parameters of the linear classification layer, and the 
function $\mathcal{M}_{\boldsymbol{\mathrm{\Theta_t}}}(\cdot): \mathbb{R}^n \to \mathbb{R}^n$ denotes 
the feature extraction layer. 
The first-order partial derivatives of $\mathcal{M}_{\boldsymbol{\mathrm{\Theta_t}}}(\cdot)$ 
with respect to each component of the vector $\mathbf{x}$ exist.

To simplify the discussion, we need to rewrite the consistency regularization loss in Equation (\ref{LU}) into the following form in the continuous space:
\begin{align}
\label{LU_se}
\mathcal{L}_u = \frac{1}{\left| \mathcal{D}_{\tau} \right|}\sum_{i=1}^{|\mathcal{D}_{\tau}|} \left\| \mathcal{P}_{\boldsymbol{\mathrm{\Theta_t}}, \boldsymbol{\beta_t}}(\mathbf{x}_{i}^u + \boldsymbol{\Delta} \mathbf{x}) - \mathcal{P}_{\boldsymbol{\mathrm{\Theta_t}}, \boldsymbol{\beta_t}}(\mathbf{x}_{i}^u) \right\|_1,
\end{align}
where $\boldsymbol{\Delta} \mathbf{x} = \Omega(\mathbf{x}) - \omega(\mathbf{x})$. Without loss of generality, 
we assume that $\omega(\mathbf{x}) = \mathbf{x}$. 
$||\cdot||_1$ represents the $L_1$ norm, which is the sum of the absolute values of the components of a vector.

We assume that a substantial number of unlabeled data points $\{\mathbf{x}_{i}^u\}_{i=0}^{+\infty}$ and various data perturbations $\{\boldsymbol{\Delta} \mathbf{x}\}$ are sampled in the continuous space of input data. Thus, in the limit, the computed consistency regularization loss can be expressed using the integral form presented in Lemma \ref{lemma1}. A detailed proof is available in the Appendix \ref{Proof_Lemma}.

\setcounter{lemma}{0}
\renewcommand{\thelemma}{4.1}

\begin{lemma}
\label{lemma1}
Equation 
(\ref{LU_se}) leads to a simplified integral expression for consistency regularization loss function:
\begin{align}
\label{LU3}
\mathcal{L}_u = \iint\limits_{\mathcal{D}} \|\nabla_\mathbf{x} \mathcal{P}_{\boldsymbol{\mathrm{\Theta_t}}, \boldsymbol{\beta_t}}\|_1 \, dV,
\end{align}
where $\mathcal{D}$ represents the continuous input data space spanned by all unlabeled data under infinite data augmentation, and $\nabla_\mathbf{x}$ represents the gradient operator with respect to the input $\mathbf{x}$.
\end{lemma}
By leveraging concepts from the PAC (Probably Approximately Correct) learning theoretical framework \cite{chapelle2009semi}, we use $\mathcal{HS}$ to represent the hypothesis space. $\mathcal{HS}$ is the set of all possible parameter values $\boldsymbol{\mathrm{\Theta_t}}$ of the feature extraction layer. Minimizing consistency regularization in Equation (\ref{LU3}) imposes the following constraints on the feature extraction layer's parameter $\boldsymbol{\mathrm{\Theta_t}}$ and the linear classification layer's parameters $\boldsymbol{\beta_t}$:
\begin{align}
\label{Consistency_argmin}
\underset{\boldsymbol{\mathrm{\Theta_t}} \in \mathcal{HS}, \, \boldsymbol{\beta_t} \in \mathbb{R}^n}{\arg\min}\mathcal{L}_u. 
\end{align}
Minimizing object in Equation (\ref{Consistency_argmin}) ensures that the model's predictions remain consistent under perturbations of the input data, thereby improving the robustness of the learned features. In practice, $\mathcal{L}_u$ is optimized iteratively to a local minimum. In the context of deep learning generalization theory, this local minimum approximates the global optimal minimum \cite{xie2020diffusion, xie2022adaptive}. Without loss of generality, 
we may assume that the global optimal minimum is 0. 
An iterative approach to minimizing consistency regularization is expressed in Equation \eqref{Consistency_argmin}, formalized as  "$\bm{\lim_{t \to +\infty} \mathcal{L}_u = 0}$" where $t$ represents the $t$-th iteration. We utilize the formal language of the PAC learning framework to describe this optimization process. This approach allows the expression to formally separate the feature extraction layer from the linear classification layer, leading to the following Theorem \ref{minichange}. Detailed proof can be found in the Appendix \ref{Proof_Theorem}.
\begin{figure*}[t] 
  \centering
  \includegraphics[width=\textwidth]{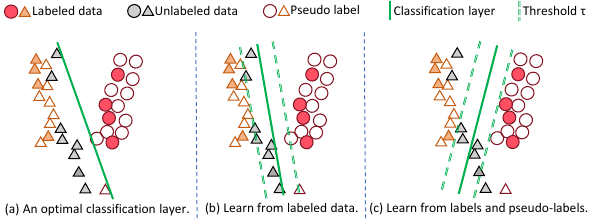}
  \caption{The example figure illustrates the data features $\mathcal{M}$ generated by the feature extraction layer $\Theta$, and explains how noisy pseudo-labels can be harmful for the linear classification layer. The "triangle $\triangle$" and "circle $\bigcirc$" represent different classes. (a) The optimal linear classification layer, which is the ideal linear classification layer constructed based on labels of all data; (b) Learning from labeled data only, where the model primarily relies on labeled data to construct the linear classification layer. Selects pseudo-labels based on the threshold $\tau$; (c) Learning from both labels and pseudo-labels, where the model uses pseudo-labels generated from unlabeled data for training. However, due to the potential noise in pseudo-labels, the linear classification layer shifts and fails to accurately distinguish between different classes.}
  \label{fig2}
  \vspace{-5mm}
\end{figure*}
\setcounter{theorem}{0}
\renewcommand{\thetheorem}{4.2}
\begin{theorem}
\label{minichange}
\(\lim_{t \to +\infty} \mathcal{L}_u = 0\)
implies finding a sequence of points \(\{\Theta_t, \boldsymbol{\beta}_t\}_{t=0}^{+\infty}\). For any $\epsilon > 0$, there exists a sequence $\{\delta_t\}_{t=0}^{+\infty}$, where $\delta_t \in [0,1]$ and $\lim_{t \to +\infty}\delta_t = 0$. There exists a $T > 0$ such that if $t > T$, the following holds:
\begin{align}
\label{eq_minichange}
P_\mathbf{x}\Big[\mathcal{P}_{\boldsymbol{\mathrm{\Theta_t}}, \boldsymbol{\beta_t}}\big(1 - \mathcal{P}_{\boldsymbol{\mathrm{\Theta_t}}, \boldsymbol{\beta_t}}\big)\big|\big|\boldsymbol{\beta}_t  \nabla_\mathbf{x} \mathcal{M}_{\Theta_t}\big|\big|_1 < \epsilon \Big] > 1 - \delta_t,
\end{align}

where $P_\mathbf{x}[\cdot]$ is the probability measure over the input data space $\mathcal{D}$. 
\end{theorem}
In the following, we analyze \textit{the pseudo-labels' impact on both the feature extraction layer and the linear classification layer}, as demonstrated by Theorem \ref{minichange} and illustrated in Figure \ref{fig2}.

First, let \(\delta\) be sufficiently small so that \(\mathcal{P}_{\boldsymbol{\mathrm{\Theta_t}}, \boldsymbol{\beta_t}}(1 - \mathcal{P}_{\boldsymbol{\mathrm{\Theta_t}}, \boldsymbol{\beta_t}})||\boldsymbol{\beta}_t \nabla_\mathbf{x}  \mathcal{M}_{\Theta_t}||_1 < \epsilon\) approximately holds over the entire \(\mathcal{D}\). Therefore, all data points satisfy \(\mathcal{P}_{\boldsymbol{\mathrm{\Theta_t}}, \boldsymbol{\beta_t}}(1 - \mathcal{P}_{\boldsymbol{\mathrm{\Theta_t}}, \boldsymbol{\beta_t}})||\boldsymbol{\beta}_t \nabla_\mathbf{x}  \mathcal{M}_{\Theta_t}||_1 < \epsilon\). The value of \(\mathcal{P}_{\boldsymbol{\mathrm{\Theta_t}}, \boldsymbol{\beta_t}}(1 - \mathcal{P}_{\boldsymbol{\mathrm{\Theta_t}}, \boldsymbol{\beta_t}})\) belongs to $[0, 0.25]$. Data points far away from the linear classification layer can hardly affect the model's optimization, because \(\mathcal{P}_{\boldsymbol{\mathrm{\Theta_t}}, \boldsymbol{\beta_t}}(1 - \mathcal{P}_{\boldsymbol{\mathrm{\Theta_t}}, \boldsymbol{\beta_t}})\) is minute. Only data points near the linear classification layer are Considered. For an easy discussion, assume \(\mathcal{P}_{\boldsymbol{\mathrm{\Theta_t}}, \boldsymbol{\beta_t}}(1 - \mathcal{P}_{\boldsymbol{\mathrm{\Theta_t}}, \boldsymbol{\beta_t}})\) is greater than a certain constant, i.e. $\tau\small(1-\tau\small)$. Thus, the topic simplifies to discussing a clear and meaningful item \(||\boldsymbol{\beta}_t \nabla_\mathbf{x}  \mathcal{M}_{\Theta_t}||_1 < \epsilon\). 

Next, we discuss the case where data points satisfy \(||\boldsymbol{\beta}_t \nabla_\mathbf{x}  \mathcal{M}_{\Theta_t}||_1 < \epsilon\). $\nabla_\mathbf{x} \mathcal{M}_{\Theta_t}$ represents the impact of minimizing pseudo-label's consistency regularization loss in Equation \ref{LU3} on the feature extraction layer. The term $\|\nabla_\mathbf{x} \mathcal{M}_{\Theta_t}\|_1$ reflects the tightness of data clusters. The larger value of $\|\nabla_\mathbf{x} \mathcal{M}_{\Theta_t}\|_1$ in a region, the sparser data cluster in that region. The sparser data cluster means poor data clustering 
ability of the feature extraction layer. Under the constraint $\|\boldsymbol{\beta}_t  \nabla_\mathbf{x} \mathcal{M}_{\Theta_t}\|_1 < \epsilon$, 
 $\nabla_\mathbf{x} \mathcal{M}_{\Theta_t}$ corresponds to the weight coefficients when $\boldsymbol{\beta}_t$ is being optimized. Hence, regions with poor clustering (large $\|\nabla_\mathbf{x} \mathcal{M}_{\Theta_t}\|_1$) have a significant impact on the linear classification layer \(\boldsymbol{\beta}_t\). These poorly clustered categories and regions are more likely to generate incorrect pseudo-labels. Therefore, pseudo-labels overall exhibit a detrimental effect on the linear classification layer.

As shown in Figure \ref{fig2}, the data clustering of "circle \(\bigcirc\)" is poor. There is an incorrect pseudo-labels in the regions where "circle \(\bigcirc\)" category is most poorly clustered. Comparing Figure \ref{fig2}(b) and Figure \ref{fig2}(c), it can be seen that pseudo-labels cause the linear classification layer \(\boldsymbol{\beta}_t\) to shift and a decrease on classification performance.

To improve the model's performance, it is essential to propose a method to eliminate the detrimental effects of pseudo-labels on the linear classification layer. Meanwhile, the method can still leverage pseudo-labels to promote clustering in the feature extraction layer to obtain tight clusters.

\subsection{LayerMatch}
To adapt to the specific characteristics of different layers, we propose the LayerMatch method. 
It includes two layer-wise pseudo-label strategies: "Grad-ReLU" and "Avg-Clustering."

\subsubsection{Grad-ReLU}

The \(\text{Grad-ReLU}\) addresses the negative impact of pseudo-labels on the model's classification layer in semi-supervised learning. 
Its core idea is to nullify the gradient influence of the unsupervised loss on the classification layer. By updating the classification layer parameters solely with the supervised loss, \(\text{Grad-ReLU}\) avoids the misleading effects of pseudo-labels.
Meanwhile, by retaining the influence of the unsupervised loss gradients in the feature extraction layer, \(\text{Grad-ReLU}\) allows the feature extraction layer to learn the clustering features of the data from pseudo-labels, effectively capturing the data distribution.

At $t$-th iteration,\(\text{Grad-ReLU}\) updates the feature extraction layer 
parameters \(\Theta_t\) using gradients from both the supervised loss the supervised loss \(\mathcal{L}_s\) and the unsupervised loss \(\mathcal{L}_u\).
For the classification layer parameters \(\beta_t\), it updates them using only the supervised loss \(\mathcal{L}_s\)
, setting the gradients from \(\mathcal{L}_u\)  to zero. This approach ensures reliable updates for the classification layer while allowing the feature extraction layer to learn from the pseudo-labels.
Under the \(\text{Grad-ReLU}\) strategy, the gradients calculation rule is given by:
\[
 \nabla_{\boldsymbol{W}}\text{Grad-ReLU}(\mathcal{L})=
\begin{cases}
    \boldsymbol{0}, & \text{if \(W\) is param. of classification layer AND \(\mathcal{L}\) $\neq$ \(\mathcal{L}_s\)},\\
    \frac{\partial \mathcal{L}}{\partial \boldsymbol{W}}, & \text{otherwise.} \\
\end{cases}
\]
Grad-ReLU effectively separates the learning tasks of different layers, allowing the model to benefit more stably from unsupervised data while avoiding the adverse effects of pseudo-labels on the classification layer.

\subsubsection{Avg-Clustering}

The Avg-Clustering strategy aims to integrate stable clustering centers to promote the formation of high-density clusters, accelerate convergence, and reduce the impact of pseudo-label errors in low-density regions. By using an exponential moving average (EMA) strategy, we integrate a stable output transformation to stabilize the clustering centers, accelerating the convergence of features in the feature extraction layer 
towards stable clustering centers. At the \(t\)-th iteration, we define the following Avg-Clustering loss function:
\begin{equation}
\label{avgclustering}
\mathcal{L}_\text{Avg-Clustering} = \frac{1}{\left| \mathcal{D}_{\tau} \right|} \sum_{i=1}^{\left| \mathcal{D}_{\tau} \right|} \mathcal{H}\Bigg(y_{\text{ul}}^{(i)}, \mathcal{P}_{\mathrm{\Theta_t},\boldsymbol{\overline{\beta}}_t}\bigg( y \,\Big|\, \Omega\left( x_{\text{ul}}^{(i)} \right) \bigg) \Bigg).
\end{equation}
The acquisition of \(\boldsymbol{\overline{\beta}}_t\) is achieved through the following formula:
\[
\begin{aligned}
    &\boldsymbol{\overline{\beta}}_t = 
    \begin{cases}
        \boldsymbol{\beta}_t,  &\text{if \(t\) mod \(N\) = 0}, \\
        m \cdot \boldsymbol{\overline{\beta}}_{t-1} + (1-m)\cdot\boldsymbol{\tilde{\beta}},  &\text{otherwise},
    \end{cases} \\
    &\boldsymbol{\tilde{\beta}} = \boldsymbol{\overline{\beta}}_{t-1} - \alpha \cdot \nabla_{\boldsymbol{\overline{\beta}}_{t-1}}\mathcal{L}_\text{Avg-Clustering}.
\end{aligned}
\]
This strategy smooths the parameter update process, ensuring that the clustering centers of the feature extraction layer remain more stable. Consequently, it enhances the model's generalization ability and accelerates its convergence speed.

The overall objective of LayerMatch is to optimize the model by integrating the supervised loss, unsupervised loss, and Avg-Clustering loss. Specifically, LayerMatch employs the Grad-ReLU strategy to separate the gradient updates for different layers and uses the Avg-Clustering strategy to stabilize the clustering centers of the feature extraction layer to achieve optimal learning outcomes.

The overall objective function of LayerMatch is given by:
\begin{equation}
\label{overall}
\mathcal{L} = \text{Grad-ReLU}(\mathcal{L}{s} + w_u \mathcal{L}{u} + w_{ac} \mathcal{L}_\text{Avg-Clustering}).
\end{equation}

\section{Experiments}
\subsection{Implementation Details}

We evaluate the efficacy of LayerMatch on several SSL image classification common benchmarks, including CIFAR-10/100 \cite{krizhevsky2009learning},  STL-10 \cite{coates2011analysis}, and ImageNet\cite{deng2009imagenet}.
Specially, CIFAR-10 contains 10 classes with 6,000 images per class while CIFAR-100 features 100 classes with 600 images per class. 
STL-10 \cite{coates2011analysis} presents unique challenges with 500 labeled samples per class and 100,000 unlabeled samples, which is regarded as a more challenging dataset because some of them are out-of-distribution.
Moreover, ImageNet-100 is a subset of ImageNet-1k Dataset, containing 100 classes and about 1350 images per class.

Following previous works~\cite{sohn2020fixmatch,xu2021dash,berthelot2021adamatch,li2024instant, wang2022usb}, we conduct experiments with two experimental frameworks: fine-tuning pre-trained Vision Transformers (ViT) and training from scratch. 
For fine-tuning pre-trained ViTs experimental framework, we conduct experiments with varying amounts of labeled data.
In detail, labeled samples per class on CIFAR-10, CIFAR-100 and STL-10 are \{1, 4, 25\}, \{2, 4, 25\} and  \{1, 4, 10\}, respectively. 
For training from scratch experimental framework, labeled samples per class on CIFAR-10, CIFAR-100, STL-10 and ImageNet-100 are \{4, 25, 400\}, \{4, 25, 100\},  \{4, 25, 100\} and \{100\}.

For fair comparison, 
we train and evaluate all methods using the unified codebase USB \cite{wang2022usb} and use the same backbones and hyperparameters as InstanT \cite{li2024instant} and FreeMatch \cite{wang2022freematch}. For fine-tuning pre-trained Vision Transformers (ViT),  
AdamW \cite{loshchilov2017decoupled} opeimizer is used. The learning rates are specifically tailored for each dataset: $5e-4$ for CIFAR-10 and CIFAR-100, and $1e-4$ for STL-10. 
We employ a cosine learning rate annealing scheme throughout the training process, maintaining a total training step count of 204,800 iterations. Both labeled and unlabeled batches are 8. For training models from scratch, the optimizer is SGD, supplemented with a momentum of 0.9. The initial learning rate is set at $0.03$. The learning rate scheme uses the cosine learning rate annealing strategy to modulate the learning rate over a total of $2^{20}$ training iterations. The experiments for CIFAR-10/100 and STL-10 are conducted on a single A100-40G GPU. The experiments for ImageNet-100 are conducted on four A100-40G GPUs. More details on the hyper-parameters are shown in Appendix \ref{Resources} and \ref{Hyperparameter}. 

In the implementation of LayerMatch, five specific parameters need to be set. 
The parameters in Avg-Clustering \ref{avgclustering}, \(n\), \(m\) and \(\alpha\), are set \(m\) to 2048, 0.999 and 5e-4, respectively. The loss weights, \(w_u\) and \(w_{ac}\) are set to $1$ in the overall objective function \ref{overall}.

\subsection{Results}

\textbf{Main Results:} 
According to the results in the Table \ref{table:main}, LayerMatch demonstrates significant performance improvements across various settings compared to SOTA methods.  Overall, LayerMatch achieves an average accuracy improvement of $2.44\%$ across different datasets. 
Specifically, on CIFAR-10 with 10 labels, LayerMatch improves accuracy by $6.62\%$. On CIFAR-100 with 400 labels, it improves accuracy by $3.69\%$. On STL-10 with 40 labels, it improves accuracy by $3.74\%$. These results indicate that LayerMatch has a significant advantage in handling various semi-supervised learning tasks, especially when labeled data is limited. Our method, LayerMatch, stands out due to its innovative focus on the distinct behavior patterns of the backbone feature network and the linear classification layer in response to pseudo-labels. Unlike previous approaches that apply a uniform pseudo-label learning strategy across all layers, LayerMatch adopts tailored pseudo-label strategies for different layers.

It is noteworthy that comparing Pseudo-Label with Only Supervised reveals that simply applying the Pseudo-Label method can harms 
the performance of ViT. This indicates that pseudo-labels are not always beneficial to the model 
due to the noise in pseudo-labels. FixMatch, which uses a pseudo-labels filter, achieves an average performance improvement of $1.3\%$ over Only Supervised. LayerMatch is a more refined pseudo-label filtering strategy that filters pseudo-labels based on the sensitivity of different layers to pseudo-label noise.

\begin{table}[t!]
\centering
\caption{Top-1 accuracy (\%) with pre-trained ViT. "Only Supervised" refers to fine-tuning the pre-trained ViT model using only labeled data. The best performance is bold and the second best performance is underlined. All results are averaged with three random seeds \{0,1,2\} and reported with 2-sigma error bar.}
\label{table:main}

\begin{adjustbox}{width=\columnwidth,center}
\begin{tabular}{@{}c|ccc|ccc|ccc|c}
\toprule Dataset & \multicolumn{3}{c|}{CIFAR-10}& \multicolumn{3}{c|}{CIFAR-100}& \multicolumn{3}{c|}{STL-10} & Average \\ 
\cmidrule(r){1-1}\cmidrule(lr){2-4}\cmidrule(lr){5-7}\cmidrule(lr){8-10}\cmidrule(lr){11-11}
 \,\# Label & \multicolumn{1}{c}{10} & \multicolumn{1}{c}{40} & \multicolumn{1}{c|}{250} & \multicolumn{1}{c}{200} & \multicolumn{1}{c}{400}  & \multicolumn{1}{c|}{2500}  & \multicolumn{1}{c}{10} & \multicolumn{1}{c}{40}  & \multicolumn{1}{c|}{100} & \multicolumn{1}{c}{394.44}  \\ 
\cmidrule(r){1-1}\cmidrule(lr){2-4}\cmidrule(lr){5-7}\cmidrule(lr){8-10}\cmidrule(lr){11-11}
\, Only Supervised & 60.52{\scriptsize $\pm$5.9}&83.83{\scriptsize $\pm$3.2}&94.55{\scriptsize $\pm$0.4}&63.82{\scriptsize $\pm$0.3}&73.60{\scriptsize $\pm$0.7}&83.11{\scriptsize $\pm$0.4}&51.61{\scriptsize $\pm$10.7}&81.02{\scriptsize $\pm$3.3}&89.15{\scriptsize $\pm$0.9}&75.69\\
 \,Pseudo-Label \cite{lee2013pseudo} & 37.65{\scriptsize $\pm$3.1} & 88.21{\scriptsize $\pm$5.3} & 95.42{\scriptsize $\pm$0.4} & 63.34{\scriptsize $\pm$2.0} & 73.13{\scriptsize $\pm$0.9} & 84.28{\scriptsize $\pm$0.1} & 30.74{\scriptsize $\pm$6.7} & 57.16{\scriptsize $\pm$4.2} & 73.44{\scriptsize $\pm$1.5} & 67.04\\
 
 \,MeanTeacher \cite{tarvainen2017mean} & 64.57{\scriptsize $\pm$4.9} & 87.15{\scriptsize $\pm$2.5} & 95.25{\scriptsize $\pm$0.5} & 59.50{\scriptsize $\pm$0.8} & 69.42{\scriptsize $\pm$0.9} & 82.91{\scriptsize $\pm$0.4}  & 42.72{\scriptsize $\pm$7.8} & 66.80{\scriptsize $\pm$3.4} & 77.71{\scriptsize $\pm$1.8} & 71.78\\ 
 \, MixMatch \cite{berthelot2019mixmatch} & 65.04{\scriptsize $\pm$2.6} & 97.16{\scriptsize $\pm$0.9} & 97.95{\scriptsize $\pm$0.1} & 60.36{\scriptsize $\pm$1.3} & 72.26{\scriptsize $\pm$0.1} & 83.84{\scriptsize $\pm$0.2}  & 10.68{\scriptsize $\pm$1.1} & 27.58{\scriptsize $\pm$16.2} &  61.85{\scriptsize $\pm$11.3} & 64.08\\
 \, VAT \cite{miyato2018virtual} & 60.07{\scriptsize $\pm$6.3} & 	93.33{\scriptsize $\pm$6.6} & 97.67{\scriptsize $\pm$0.2} & 65.89{\scriptsize $\pm$1.8} & 75.33{\scriptsize $\pm$0.4} & 83.42{\scriptsize $\pm$0.4} & 20.57{\scriptsize $\pm$4.4} & 65.18{\scriptsize $\pm$7.0} & 80.94{\scriptsize $\pm$1.0} & 71.38\\
 \, UDA \cite{xie2020unsupervised} & 78.76{\scriptsize $\pm$3.6} & 97.92{\scriptsize $\pm$0.2} & 97.96{\scriptsize $\pm$0.1} & 	65.49{\scriptsize $\pm$1.6} & 75.85{\scriptsize $\pm$0.6} & 83.81{\scriptsize $\pm$0.2} & 48.37{\scriptsize $\pm$4.3} & 79.67{\scriptsize $\pm$4.9} & 89.46{\scriptsize $\pm$1.0} & 79.70\\
  \,FixMatch \cite{sohn2020fixmatch} & 66.50{\scriptsize $\pm$15.1} & 97.44{\scriptsize $\pm$0.9} & 97.95{\scriptsize $\pm$0.1} & 65.29{\scriptsize $\pm$1.4} & 75.52{\scriptsize $\pm$0.1} & 83.98{\scriptsize $\pm$0.1} & 40.13{\scriptsize $\pm$3.4} & 77.72{\scriptsize $\pm$4.4} & 88.41{\scriptsize $\pm$1.6} & 76.99 \\
 \, FlexMatch \cite{zhang2021flexmatch} & 	70.54{\scriptsize $\pm$9.6} & 97.78{\scriptsize $\pm$0.3} & 97.88{\scriptsize $\pm$0.2} & 63.76{\scriptsize $\pm$0.9} & 74.01{\scriptsize $\pm$0.5} & 83.72{\scriptsize $\pm$0.2}  & 60.63{\scriptsize $\pm$12.9} & 78.17{\scriptsize $\pm$3.7} & 89.54{\scriptsize $\pm$1.3} & 79.56\\
 \, Dash \cite{xu2021dash} & 74.35{\scriptsize $\pm$4.5}  & 96.63{\scriptsize $\pm$2.0} & 97.90{\scriptsize $\pm$0.3} & 63.33{\scriptsize $\pm$0.4} & 74.54{\scriptsize $\pm$0.2} & 84.01{\scriptsize $\pm$0.2}  & 41.06{\scriptsize $\pm$4.4} & 78.03{\scriptsize $\pm$3.9} & 89.56{\scriptsize $\pm$2.0} & 77.71 \\
 
 \,SoftMatch \cite{chen2023softmatch} & 80.03{\scriptsize $\pm$9.0} & \underline{97.95{\scriptsize $\pm$1.0}} & \underline{98.14{\scriptsize $\pm$0.1}} & 70.57{\scriptsize $\pm$1.0} & 78.21{\scriptsize $\pm$1.0} &\textbf{87.68{\scriptsize $\pm$0.2}}& 65.10{\scriptsize $\pm$9.0} & 83.70{\scriptsize $\pm$4.0}&92.06{\scriptsize $\pm$1.8}& 83.72 \\
\,FreeMatch \cite{wang2022freematch} & 77.21{\scriptsize $\pm$5.0} & \textbf{98.11{\scriptsize $\pm$0.1}} & \textbf{98.19{\scriptsize $\pm$0.1}} & \underline{76.29{\scriptsize $\pm$2.0}} & \underline{79.38{\scriptsize $\pm$0.1}} &\underline{87.61{\scriptsize $\pm$0.2}}& 62.30{\scriptsize $\pm$14.0} & 84.96{\scriptsize $\pm$3.0}&\underline{92.08{\scriptsize $\pm$1.6}} & 84.01 \\ 
\,AdaMatch \cite{berthelot2021adamatch} & 85.15{\scriptsize $\pm$20.4} & 97.94{\scriptsize $\pm$0.1} & 97.92{\scriptsize $\pm$0.1} & 73.61{\scriptsize $\pm$0.1} & 78.59{\scriptsize $\pm$0.4} & 84.49{\scriptsize $\pm$0.1} & 68.17{\scriptsize $\pm$7.7}& 83.50{\scriptsize $\pm$4.2} & 89.25{\scriptsize $\pm$1.5} & 84.29\\
\,InstanT \cite{li2024instant} & \underline{87.32{\scriptsize $\pm$10.2}} & 97.93{\scriptsize $\pm$0.1} & 98.08{\scriptsize $\pm$0.1} & 74.17{\scriptsize $\pm$0.3} & 78.80{\scriptsize $\pm$0.4} & 84.28{\scriptsize $\pm$0.5} &  \textbf{69.39{\scriptsize $\pm$7.4}} & \underline{85.09{\scriptsize $\pm$2.8}} & 89.35{\scriptsize $\pm$1.9} & \underline{84.93} \\ 
\cmidrule(r){1-1}\cmidrule(lr){2-4}\cmidrule(lr){5-7}\cmidrule(lr){8-10}\cmidrule(lr){11-11}
\,LayerMatch (Our) & \textbf{93.94{\scriptsize $\pm$3.9}}&97.65{\scriptsize $\pm$0.1}&97.78{\scriptsize $\pm$0.1}&\textbf{76.95{\scriptsize $\pm$1.4}}&\textbf{83.07{\scriptsize $\pm$1.1}}&87.34{\scriptsize $\pm$0.2}&\underline{68.52{\scriptsize $\pm$8.0}}&\textbf{88.83{\scriptsize $\pm$3.0}}&\textbf{92.25{\scriptsize $\pm$1.2}} & \textbf{87.37}\\
\bottomrule
\end{tabular}
\end{adjustbox}
\end{table}

\begin{table}[t]
\vspace{-.21in}

\centering
\caption{Top-1 accuracy (\%) with training from scratch. The best performance is bold and the second best performance is underlined. All results are averaged with three random seeds \{0,1,2\} and reported with 2-sigma error bar.}
\label{table:main2}

\begin{adjustbox}{width=\columnwidth,center}
\begin{tabular}{@{}c|ccc|ccc|ccc|c}
\toprule Dataset & \multicolumn{3}{c|}{CIFAR-10}& \multicolumn{3}{c|}{CIFAR-100}& \multicolumn{3}{c|}{STL-10} & Average \\ 
\cmidrule(r){1-1}\cmidrule(lr){2-4}\cmidrule(lr){5-7}\cmidrule(lr){8-10}\cmidrule(lr){11-11}
 \,\# Label & \multicolumn{1}{c}{40} & \multicolumn{1}{c}{250} & \multicolumn{1}{c|}{4000} & \multicolumn{1}{c}{400} & \multicolumn{1}{c}{2500}  & \multicolumn{1}{c|}{10000}  & \multicolumn{1}{c}{40} & \multicolumn{1}{c}{250}  & \multicolumn{1}{c|}{1000} & \multicolumn{1}{c}{2053.33}  \\ 
\cmidrule(r){1-1}\cmidrule(lr){2-4}\cmidrule(lr){5-7}\cmidrule(lr){8-10}\cmidrule(lr){11-11}
\, Pseudo-Label \cite{lee2013pseudo} & 25.39{\scriptsize $\pm$0.3} & 53.51{\scriptsize $\pm$2.2} & 84.92{\scriptsize $\pm$0.2} & 12.55{\scriptsize $\pm$0.9} & 42.26{\scriptsize $\pm$0.3} & 63.45{\scriptsize $\pm$0.2} & 25.32{\scriptsize $\pm$1.0} & 44.55{\scriptsize $\pm$2.4} & 67.36{\scriptsize $\pm$0.7} &46.59\\
\, UDA \cite{xie2020unsupervised} & 91.60{\scriptsize $\pm$1.5} & 94.44{\scriptsize $\pm$0.3} & 95.55{\scriptsize $\pm$0.0} & 40.60{\scriptsize $\pm$1.8} & 64.79{\scriptsize $\pm$0.8} & 72.13{\scriptsize $\pm$0.2} & 62.58{\scriptsize $\pm$8.0} & 90.28{\scriptsize $\pm$1.2} & 93.36{\scriptsize $\pm$0.2} &78.37\\
\, FixMatch \cite{sohn2020fixmatch} & 91.45{\scriptsize $\pm$1.7} & {94.87{\scriptsize $\pm$0.1}} & 95.51{\scriptsize $\pm$0.1} & 45.08{\scriptsize $\pm$3.4} & 65.63{\scriptsize $\pm$0.4} & 71.65{\scriptsize $\pm$0.3} & 62.00{\scriptsize $\pm$7.1} & 90.14{\scriptsize $\pm$1.0} & 93.81{\scriptsize $\pm$0.3} &78.90\\
\, FlexMatch \cite{zhang2021flexmatch} & 94.53{\scriptsize $\pm$0.4} & 94.85{\scriptsize $\pm$0.1} & {95.62{\scriptsize $\pm$0.1}} & 47.81{\scriptsize $\pm$1.4} & 66.11{\scriptsize $\pm$0.4} & {72.48{\scriptsize $\pm$0.2}} & 70.81{\scriptsize $\pm$4.2} & 90.21{\scriptsize $\pm$1.1} & 93.55{\scriptsize $\pm$0.4} &80.66\\
\, Dash \cite{xu2021dash} & 84.67{\scriptsize $\pm$4.3} & 94.78{\scriptsize $\pm$0.3} & 95.54{\scriptsize $\pm$0.1} & 45.26{\scriptsize $\pm$2.6} & 65.51{\scriptsize $\pm$0.1} & 72.10{\scriptsize $\pm$0.3} & 65.48{\scriptsize $\pm$4.3} & 90.93{\scriptsize $\pm$0.9} & 93.61{\scriptsize $\pm$0.6} &78.65\\
\, FreeMatch \cite{wang2022freematch} & \textbf{94.98{\scriptsize $\pm$0.1}} & \textbf{95.16{\scriptsize $\pm$0.1}}& \textbf{95.88{\scriptsize $\pm$0.2}}&51.06{\scriptsize $\pm$2.0}&66.59{\scriptsize $\pm$0.2}&\underline{72.93{\scriptsize $\pm$0.1}}&\underline{83.57{\scriptsize $\pm$2.7}}&{91.31{\scriptsize $\pm$0.2}}&\underline{94.15{\scriptsize $\pm$0.2}} & \underline{82.85}\\
\, AdaMatch \cite{berthelot2021adamatch} & 94.64{\scriptsize $\pm$0.0} & 94.76{\scriptsize $\pm$0.1} & 95.46{\scriptsize $\pm$0.1} & {52.02{\scriptsize $\pm$1.7}} & {66.36{\scriptsize $\pm$0.7}} & {72.32{\scriptsize $\pm$0.2}} & {80.05{\scriptsize $\pm$5.2}} & \underline{91.41{\scriptsize $\pm$0.4}} & {93.99{\scriptsize $\pm$0.0}} &{82.33}\\
\, InstanT \cite{li2024instant} & {94.83{\scriptsize $\pm$0.1}} & 94.72{\scriptsize $\pm$0.2} & {95.57{\scriptsize $\pm$0.0}} & \underline{53.94{\scriptsize $\pm$1.8}} & \underline{67.09{\scriptsize $\pm$0.0}} & 72.30{\scriptsize $\pm$0.4} & - & - & - &-\\
\cmidrule(r){1-1}\cmidrule(lr){2-4}\cmidrule(lr){5-7}\cmidrule(lr){8-10}\cmidrule(lr){11-11}
\, LayerMatch (Our) & \underline{94.92{\scriptsize $\pm$0.1}} & \underline{95.03{\scriptsize $\pm$0.1}}& \underline{95.67{\scriptsize $\pm$0.1}}&\textbf{60.87{\scriptsize $\pm$0.7}}&\textbf{73.96{\scriptsize $\pm$0.7}}&\textbf{78.39{\scriptsize $\pm$0.2}}&\textbf{84.07{\scriptsize $\pm$4.0}}&\textbf{92.68{\scriptsize $\pm$0.1}}&\textbf{94.58{\scriptsize $\pm$0.0}}&\textbf{85.57}\\
\bottomrule
\end{tabular}
\end{adjustbox}
\vspace{-.3in}
\end{table}

\textbf{Training from Scratch:} To more comprehensively evaluate LayerMatch, experiment results in 
the training from scratch setting are shown in Table \ref{table:main2}. Table \ref{table:main2} takes the previous state-of-the-art methods, top average performance methods in Table \ref{table:main}, and some classic approaches as benchmarks.

\begin{wraptable}{r}{0.46\textwidth}
\vspace{-.20in}
    \centering
    \caption{Top-1 accuracy (\%) and f-1 score (\%) on ImageNet-100 with 100 labels per class, random seed \{0\}.}
    \vspace{-.095in}
    \label{tab-imagenet}
    \resizebox{0.43\textwidth}{!}{%
    \begin{tabular}{c|cc}
        \toprule
        Methods & Top-1 Acc & F-1 Score \\ \midrule
        \, FixMatch \cite{sohn2020fixmatch}  & 66.24 & 65.59\\
        \, FreeMatch \cite{wang2022freematch} & 65.78 & 65.29\\
        \, AdaMatch \cite{berthelot2021adamatch} & 68.60 & 68.22 \\
        \, InstanT \cite{li2024instant} & 69.94 & 69.72 \\
        \, LayerMatch (Our) & \textbf{72.54} & \textbf{72.12} \\\bottomrule
    \end{tabular}
    }
\vspace{-.25in}
\end{wraptable}

Based on the results in the Table \ref{table:main2}, LayerMatch demonstrates a significant performance improvement.  Overall, LayerMatch achieves an average accuracy improvement of $2.72\%$ across different datasets. Especiall, on CIFAR-100, LayerMatch improves accuracy by $6.93\%$ with 400 labels, $6.87\%$ with 2500 labels, and $5.46\%$ with 10000 labels.

In our evaluation of semi-supervised learning methods on the ImageNet-100 dataset with 100 labels per class, we observed significant variations in performance among the tested methods (Table \ref{tab-imagenet}). Our proposed method, LayerMatch, 
demonstrates superior performance, achieving the highest Top-1 accuracy of $72.54\%$ and an F-1 score of $72.12\%$. This represents a notable improvement over other state-of-the-art techniques. For instance, InstanT, which is 
the second-best performing method, reaches 
a Top-1 accuracy of $69.94\%$ and an F-1 score of $69.72\%$. LayerMatch improves 
the Top-1 accuracy by $2.60\%$ and the F-1 score by $2.40\%$ compared to SOTA method InstanT. These results highlight LayerMatch's superior performance on this challenging dataset.


The superior performance of LayerMatch can be attributed to its innovative approach that optimizes the differential behavior between the feature extraction layers and the classification layer in response to pseudo-labels. This differential strategy enhances the model's accuracy and its ability to generalize from limited labeled data and achieve the SOTA.

\textbf{Limitations:} Although LayerMatch demonstrates superior overall performance, it has shortcomings in certain settings. For example, in Table \ref{table:main}, its performance on CIFAR-10 with 40 and 250 labels is not as good as FreeMatch. The potential reasons for this could be as follows. 
In such cases, the model's accuracy exceeds $97\%$, so the quality of the generated pseudo-labels is extremely high with very little noise and pseudo-labels are also beneficial for the linear classification layer. Therefore, the limitation of LayerMatch is that its design, which directly sets the pseudo-label gradient of the linear classification layer to zero, is relatively simplistic. Future work should focus on designing a method that selects pseudo-labels for self-training different layers based on the learning state of each layer.

\begin{wrapfigure}{r}{0.6\textwidth}
  \centering
  \vspace{-14mm}
  \includegraphics[width=0.6\textwidth]{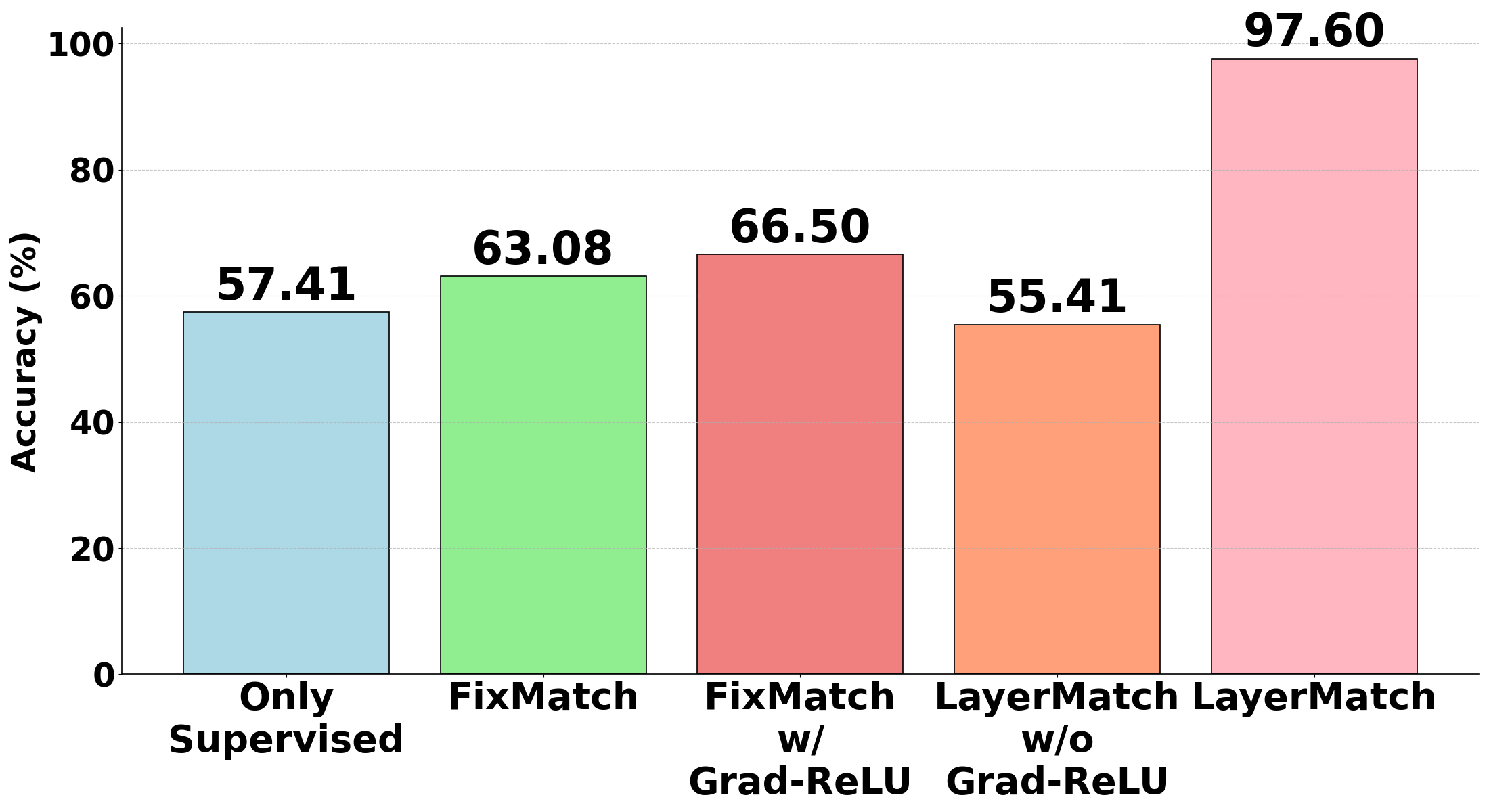}
  \vspace{-6mm}
  \caption{Ablation study of Grad-ReLU on CIFAR-10 (10) with pre-trained ViT.}
  \label{fig:ablation1}
  \vspace{-4mm}
\end{wrapfigure}

\subsection{Ablation Study}

We conduct ablation studies to the individual and combined effects of the Grad-ReLU and Avg-Clustering strategies. 
"CIFAR-10 (10)" represents experimental settings with 10 labels on CIFAR-10.

\begin{wraptable}{r}{0.45\textwidth}
\vspace{-4.5mm}
    \centering
    \caption{Ablation study of Avg-Clustering on CIFAR-10 (10) with pre-trained ViT.}
    \vspace{-2.7mm}
    \label{ablation2}
    \resizebox{0.45\textwidth}{!}{%
    \begin{tabular}{c|cc|c}
        \toprule
        Method  & Grad-ReLU & Avg-Clustering & Acc. (\%)\\ \midrule
       LayerMatch & \xmark & \cmark & 55.41\\
        LayerMatch & \cmark & \xmark & 80.45\\
       LayerMatch  & \cmark & \cmark  &  \textbf{97.60}\\\bottomrule
    \end{tabular}
    }
\end{wraptable}

In Figure \ref{fig:ablation1}, applying the Grad-ReLU strategy alone significantly improves the model’s accuracy compared to FixMatch, which yields 
a Top-1 accuracy of $63.08\%$. Incorporating Grad-ReLU into FixMatch results in an accuracy of $66.50\%$. This shows that pseudo-labeled data can be detrimental to the linear classification layer in FixMatch. When LayerMatch is tested without Grad-ReLU strategies, the accuracy slightly decreases to $55.41\%$. 
It underscores the necessity of our design 
in layer-wise managing less reliable pseudo-labels.

In Table \ref{ablation2}, the most notable outcome is observed 
when both Grad-ReLU and Avg-Clustering were enabled within LayerMatch, which pushed the accuracy to an impressive $97.60\%$. This underscores the synergistic effect of combining these strategies to layer-wisely utilize both labeled and unlabeled data, optimizing the learning process extensively.

In Table \ref{ablation3}, varying the parameter $N$ in Avg-Clustering shows 
that a moderate setting ($N = 2048$) yields the best results, achieving the same peak accuracy of $97.60\%$. Very high ($N = 204800$) or low settings ($N = 1$) led to less 
accuracy, indicating an optimal balance requires 
in the update dynamics to maximize learning efficacy. 

\begin{wraptable}{r}{0.35\textwidth}
\vspace{-4.7mm}
    \centering
    \caption{Ablation study of $N$ of Avg-Clustering on CIFAR-10 (10) with pre-trained ViT.}
    \vspace{-2.7mm}
    \label{ablation3}
    \resizebox{0.28\textwidth}{!}{%
    \begin{tabular}{c|c|c}
        \toprule
        Method  & N  & Acc. (\%) \\ \midrule
        LayerMatch & 1 & 87.46\\
         LayerMatch & 2048 & \textbf{97.60}\\
         LayerMatch & 204800  &  82.45\\\bottomrule
    \end{tabular}
    }
\vspace{-10mm}
\end{wraptable}

These results demonstrate the perfect synergy between the two strategies in LayerMatch. Grad-ReLU, designed specifically for the linear classification layer, and Avg-Clustering, tailored for the feature extraction layer, mutually reinforce each other, highlighting the remarkable effectiveness of the layer-wisely design.


\section{Conclusion}

This paper revisits the role of pseudo-labels in different layers, emphasizing that the feature extraction layer learns clustering features from pseudo-labels.
However, the linear classification layer can be misled by poorly clustered pseudo-labeled data.
To address these issues, the LayerMatch method is introduced 
, which aligns pseudo-label usage with the characteristics of each layer. 
Within LayerMatch, the Grad-ReLU strategy is employed to prevent the negative impact of pseudo-labels on the classification layer.
Moreover, the Avg-Clustering strategy is used to stabilize the clustering centers in the feature extraction layer.
This work aims to inspire researchers to design 
layer-wisely adaptive pseudo-label algorithms.

\bibliography{main}
\bibliographystyle{plain}
\newpage
\appendix
\setcounter{lemma}{0}
\renewcommand{\thelemma}{4.1}
\section{Proof of Lemma 4.1.}
\label{Proof_Lemma}
\begin{lemma}
Equation (\ref{LU_se}) leads to a simplified integral expression for consistency regularization loss function:
\begin{align*}
\mathcal{L}_u = \iint\limits_{\mathcal{D}} \|\nabla_\mathbf{x} \mathcal{P}_{\boldsymbol{\mathrm{\Theta_t}}, \boldsymbol{\beta_t}}\|_1 \, dV,
\end{align*}
where $\mathcal{D}$ represents the continuous input data space spanned by all unlabeled data under infinite data augmentation, and $\nabla_\mathbf{x}$ represents the gradient operator with respect to the input $\mathbf{x}$.
\end{lemma}
\begin{proof}
By the Mean Value Theorem, we have
\begin{align*}
\mathcal{P}_{\boldsymbol{\mathrm{\Theta_t}}, \boldsymbol{\beta_t}}(\mathbf{x}+\boldsymbol{\Delta} \mathbf{x})-\mathcal{P}_{\boldsymbol{\mathrm{\Theta_t}}, \boldsymbol{\beta_t}}(\mathbf{x}) = \nabla \mathcal{P}_{\boldsymbol{\mathrm{\Theta_t}}, \boldsymbol{\beta_t}}(\mathrm{\xi}) \cdot \boldsymbol{\Delta} \mathbf{x},
\end{align*}
where $\mathrm{\xi}$ is a point on the line segment between $\mathbf{x}$ and $\mathbf{x}+\boldsymbol{\Delta} \mathbf{x}$.

When $\boldsymbol{\Delta} \mathbf{x}$ is sufficiently small, we have the approximation
\begin{align*}
\mathcal{P}_{\boldsymbol{\mathrm{\Theta_t}}, \boldsymbol{\beta_t}}(\mathbf{x}+\boldsymbol{\Delta} \mathbf{x})-\mathcal{P}_{\boldsymbol{\mathrm{\Theta_t}}, \boldsymbol{\beta_t}}(\mathbf{x}) = \nabla \mathcal{P}_{\boldsymbol{\mathrm{\Theta_t}}, \boldsymbol{\beta_t}}(\mathbf{x}) \cdot \boldsymbol{\Delta} \mathbf{x}.
\end{align*}
Partition the unlabeled data space $\mathcal{D}$ into small hypercubes with equidistant intervals $\Delta x$, and consider the vertices of these cubes as the set $\mathcal{D}_{\tau}$. Using the equation above, we can reformulate Equation \ref{LU_se} as follows:
\begin{align*}
\mathcal{L}_u &= \frac{1}{\left| \mathcal{D}_{\tau} \right|}\sum_{i=1}^{|\mathcal{D}_{\tau}|} \left\| \mathcal{P}_{\boldsymbol{\mathrm{\Theta_t}}, \boldsymbol{\beta_t}}(\mathbf{x}_{i}^u + \boldsymbol{\Delta} \mathbf{x}) - \mathcal{P}_{\boldsymbol{\mathrm{\Theta_t}}, \boldsymbol{\beta_t}}(\mathbf{x}_{i}^u) \right\|_1\\
&=\frac{1}{\left| \mathcal{D}_{\tau} \right|}\sum_{i=1}^{|\mathcal{D}_{\tau}|} \left\| \nabla\mathcal{P}_{\boldsymbol{\mathrm{\Theta_t}}, \boldsymbol{\beta_t}}(\mathbf{x}_{i}^u)  \cdot \boldsymbol{\Delta} \mathbf{x} \right\|_1\\
&=h\cdot\sum_{i=1}^{|\mathcal{D}_{\tau}|}||\nabla\mathcal{P}_{\boldsymbol{\mathrm{\Theta_t}}, \boldsymbol{\beta_t}}(\mathbf{x}_{i}^u)||_1,
\end{align*}
Here, $h=\boldsymbol{\Delta} \mathbf{x}_1=\boldsymbol{\Delta} \mathbf{x}_2=\dots=\boldsymbol{\Delta} \mathbf{x}_n$. When $h \to 0$,  we can continue with:
\begin{align*}
\mathcal{L}_u &= \lim\limits_{h \to 0} h\cdot\sum_{i=1}^{|\mathcal{D}_{\tau}|}||\nabla\mathcal{P}_{\boldsymbol{\mathrm{\Theta_t}}, \boldsymbol{\beta_t}}(\mathbf{x}_{i}^u)||_1=\iint\limits_{\mathcal{D}} \|\nabla_\mathbf{x} \mathcal{P}_{\boldsymbol{\mathrm{\Theta_t}}, \boldsymbol{\beta_t}}\|_1 \, dV.
\end{align*}
Hence we proved Lemma 4.1.
\end{proof}
\setcounter{theorem}{0}
\renewcommand{\thetheorem}{4.2}
\section{Proof of Theorem 4.2.}
\label{Proof_Theorem}
\begin{theorem}
\(\lim_{t \to +\infty} \mathcal{L}_u = 0\)
implies finding a sequence of points \(\{\Theta_t, \boldsymbol{\beta}_t\}_{t=0}^{+\infty}\). For any $\epsilon > 0$, there exists a sequence $\{\delta_t\}_{t=0}^{+\infty}$, where $\delta_t \in [0,1]$ and $\lim_{t \to +\infty}\delta_t = 0$. There exists a $T > 0$ such that if $t > T$, the following holds:
\begin{align*}
P_\mathbf{x}\Big[\mathcal{P}_{\boldsymbol{\mathrm{\Theta_t}}, \boldsymbol{\beta_t}}\big(1 - \mathcal{P}_{\boldsymbol{\mathrm{\Theta_t}}, \boldsymbol{\beta_t}}\big)\big|\big|\boldsymbol{\beta}_t  \nabla_\mathbf{x} \mathcal{M}_{\Theta_t}\big|\big|_1 < \epsilon \Big] > 1 - \delta_t,
\end{align*}

where $P_\mathbf{x}[\cdot]$ is the probability measure over the input data space $\mathcal{D}$. 
\end{theorem}
\begin{proof}
We have $\lim_{t \to +\infty} \mathcal{L}_u = 0$ and $\|\nabla_\mathbf{x} \mathcal{P}_{\boldsymbol{\mathrm{\Theta_t}}, \boldsymbol{\beta_t}}\|_1 \geq 0$. By the fundamental theorem of calculus, we know:
\begin{align*}
\lim_{t \to +\infty} P_\mathbf{x}\big[\|\nabla_\mathbf{x} \mathcal{P}_{\boldsymbol{\mathrm{\Theta_t}}, \boldsymbol{\beta_t}}\|_1=0\big] = 1.
\end{align*}

By the chain rule of differentiation, we have
\begin{align*}
\nabla_\mathbf{x} \mathcal{P}_{\boldsymbol{\mathrm{\Theta_t}}, \boldsymbol{\beta_t}} = \mathcal{P}_{\boldsymbol{\mathrm{\Theta_t}}, \boldsymbol{\beta_t}}\big(1 - \mathcal{P}_{\boldsymbol{\mathrm{\Theta_t}}, \boldsymbol{\beta_t}}\big)\boldsymbol{\beta}_t  \nabla_\mathbf{x} \mathcal{M}_{\Theta_t}.
\end{align*}
Thus, we have
\begin{align*}
\lim_{t \to +\infty} P_\mathbf{x}\Big[\mathcal{P}_{\boldsymbol{\mathrm{\Theta_t}}, \boldsymbol{\beta_t}}\big(1 - \mathcal{P}_{\boldsymbol{\mathrm{\Theta_t}}, \boldsymbol{\beta_t}}\big)||\boldsymbol{\beta}_t  \nabla_\mathbf{x} \mathcal{M}_{\Theta_t}||_1=0\Big] = 1.
\end{align*}
According to the definition of the limit, the above statement is equivalently expressed as:

For any $\epsilon > 0$, there exists a sequence $\{\delta_t\}_{t=0}^{+\infty}$, where $\delta_t \in [0,1]$ and $\lim_{t \to +\infty}\delta_t = 0$. There exists a $T > 0$ such that if $t > T$, the following holds:
\begin{align*}
P_\mathbf{x}\Big[\mathcal{P}_{\boldsymbol{\mathrm{\Theta_t}}, \boldsymbol{\beta_t}}\big(1 - \mathcal{P}_{\boldsymbol{\mathrm{\Theta_t}}, \boldsymbol{\beta_t}}\big)\big|\big|\boldsymbol{\beta}_t  \nabla_\mathbf{x} \mathcal{M}_{\Theta_t}\big|\big|_1 < \epsilon \Big] > 1 - \delta_t.
\end{align*}
Hence we proved Theorem 4.2.
\end{proof}

\section{Experiments Compute Resources}
\label{Resources}
The experiments for CIFAR-10/100 and STL-10 are conducted on a single A100-40G GPU. The experiments for ImageNet-100 are conducted on four A100-40G GPUs. 31 A100-40G days are spend on fine-tuning pre-trained ViT setting. 175 A100-40G days are spend on training from scratch setting. The initial exploratory experiments are conducted on CIFAR-10, which are not resource-intensive, taking 8 A100-40G days.

\section{Broader Impacts}
This paper is foundational research on semi-supervised classification problems and does not involve specific applications. In the foreseeable future, the impact is expected to be limited to the field of machine learning theoretical research, aiming to advance the understanding and innovative design of methods for semi-supervised classification problems. Currently, no negative impacts have been observed.

\section{Hyperparameter setting}
The main experiment for fine-tuning pre-trained ViT in Table \ref{table:main} and all ablation experiments used the hyperparameters from Table \ref{vith1}. The experiment for training from scratch in Table \ref{table:main2} and Table \ref{tab-imagenet} utilized the hyperparameters from Table \ref{h2}.
\label{Hyperparameter}

\begin{table}[!htbp]
\centering
\caption{Hyper-parameters of experiments using pre-trained ViT.}
\resizebox{0.75\textwidth}{!}{
\begin{tabular}{cccc}\toprule
Dataset & CIFAR-10 & CIFAR-100 & STL-10 \\\cmidrule(r){1-1} \cmidrule(lr){2-2}\cmidrule(lr){3-3}\cmidrule(l){4-4}
Model    &  ViT-T-P2-32 & ViT-S-P4-32 & ViT-B-P16-96\\\cmidrule(r){1-1} \cmidrule(l){2-4}
Weight Decay&  \multicolumn{3}{c}{5e-4}\\\cmidrule(r){1-1} \cmidrule(l){2-4}
Labeled Batch size & \multicolumn{3}{c}{8}\\\cmidrule(r){1-1} \cmidrule(l){2-4}
Unlabeled Batch size & \multicolumn{3}{c}{8}\\\cmidrule(r){1-1} \cmidrule(l){2-4}
Learning Rate & 5e-4 & 5e-4 & 1e-4 \\\cmidrule(r){1-1} \cmidrule(l){2-4}
Layer Decay Rate & 0.5 & 0.5 & 0.95 \\\cmidrule(r){1-1} \cmidrule(l){2-4}
Scheduler & \multicolumn{3}{c}{$\eta = \eta_0 \cos(\frac{7\pi k}{16K})$} \\ \cmidrule(r){1-1} \cmidrule(l){2-4}
AdamW momentum & \multicolumn{3}{c}{0.9}\\\cmidrule(r){1-1} \cmidrule(l){2-4}
Model EMA Momentum & \multicolumn{3}{c}{0.0}\\\cmidrule(r){1-1} \cmidrule(l){2-4}
Prediction EMA Momentum & \multicolumn{3}{c}{0.999}\\\cmidrule(r){1-1} \cmidrule(l){2-4}
Threshold & \multicolumn{3}{c}{Self-adaptive Thresholding \cite{wang2022freematch}} \\\cmidrule(r){1-1} \cmidrule(l){2-4}
Weak Augmentation & \multicolumn{3}{c}{Random Crop, Random Horizontal Flip} \\\cmidrule(r){1-1} \cmidrule(l){2-4}
Strong Augmentation & \multicolumn{3}{c}{RandAugment \cite{cubuk2020randaugment}} \\
\bottomrule
\end{tabular}
}
\label{vith1}
\end{table}
\newpage
\begin{table}[!htbp]
\centering
\caption{Hyper-parameters of experiments training from scratch.}
\label{h2}
\begin{adjustbox}{width=0.85\columnwidth, center}
\begin{tabular}{cccccc}\toprule
Dataset &  CIFAR-10 & CIFAR-100 & STL-10 & ImageNet-100 \\\cmidrule(r){1-1} \cmidrule(lr){2-2}\cmidrule(lr){3-3}\cmidrule(lr){4-4}\cmidrule(l){5-5}
Model    &  WRN-28-2 & WRN-28-8 & WRN-37-2 & ResNet-50 \\\cmidrule(r){1-1} \cmidrule(l){2-5}
Weight Decay&  5e-4  & 1e-3 & 5e-4 & 3e-4\\\cmidrule(r){1-1} \cmidrule(l){2-5}
Unlabeled Batch size & \multicolumn{4}{c}{448}\\\cmidrule(r){1-1} \cmidrule(l){2-5}
Learning Rate & \multicolumn{4}{c}{0.03}\\\cmidrule(r){1-1} \cmidrule(l){2-5}
Scheduler & \multicolumn{4}{c}{$\eta = \eta_0 \cos(\frac{7\pi k}{16K})$} \\\cmidrule(r){1-1} \cmidrule(l){2-5}
SGD Momentum & \multicolumn{4}{c}{0.9}\\\cmidrule(r){1-1} \cmidrule(l){2-5}
Model EMA Momentum & \multicolumn{4}{c}{0.999}\\\cmidrule(r){1-1} \cmidrule(l){2-5}
Prediction EMA Momentum & \multicolumn{4}{c}{0.999}\\\cmidrule(r){1-1} \cmidrule(l){2-5}
Threshold & \multicolumn{4}{c}{Self-adaptive Thresholding \cite{wang2022freematch}} \\\cmidrule(r){1-1} \cmidrule(l){2-5}
Weak Augmentation & \multicolumn{4}{c}{Random Crop, Random Horizontal Flip}\\\cmidrule(r){1-1} \cmidrule(l){2-5}
Strong Augmentation & \multicolumn{4}{c}{RandAugment \cite{cubuk2020randaugment}} \\
\bottomrule
\end{tabular}
\end{adjustbox}
\end{table}

\end{document}